\documentclass[sigconf]{acmart}

\AtBeginDocument{%
  \providecommand\BibTeX{{%
    \normalfont B\kern-0.5em{\scshape i\kern-0.25em b}\kern-0.8em\TeX}}}

\usepackage{hyperref}
\usepackage{url}
\usepackage{graphicx}
\usepackage{multirow}
\usepackage{amsmath,bm}
\usepackage{algorithmicx}
\usepackage[ruled,vlined]{algorithm2e}
\usepackage{algpseudocode}
\usepackage[autostyle]{csquotes}

\usepackage{caption}
\usepackage{subcaption}

\newcommand{\paragraStartHighlight}[1]{\noindent\textbf{#1}}
\newcommand{\adagnn}{AdaGNN}

\usepackage{amsthm}
\newtheorem{proposition}{Proposition}
\usepackage{array}
\graphicspath {{figures/}}

\DeclareMathOperator*{\argmin}{arg\,min}
\DeclareMathOperator{\enc}{\text{\textbf{ENC}}}
\DeclareMathOperator{\dec}{\text{\textbf{DEC}}}
\DeclareMathOperator{\agg}{\text{\textbf{AGGREGATE}}}
\DeclareMathOperator{\combine}{\text{\textbf{COMBINE}}}
\begin{document}

%%
%% The "title" command has an optional parameter,
%% allowing the author to define a "short title" to be used in page headers.
\title{AdaGNN: A multi-modal latent representation meta-learner for GNNs based on AdaBoosting}

%%
%% The "author" command and its associated commands are used to define
%% the authors and their affiliations.
%% Of note is the shared affiliation of the first two authors, and the
%% "authornote" and "authornotemark" commands
%% used to denote shared contribution to the research.
\author{Qinyi Zhu}
%\authornote{Both authors contributed equally to this research.}
%\email{qinzhu@linkedin.com}
\email{qinzhu@berkeley.edu}
\affiliation{%
  \institution{University of California, Berkeley}
  \streetaddress{200 California Hall}
  \city{Berkeley}
  \state{California}
  \country{USA}
  \postcode{94720}
}

\author{Yiou Xiao}
%\authornotemark[1]
\email{yixiao@linkedin.com}
\affiliation{%
  \institution{LinkedIn}
  \streetaddress{1000 W Maude Ave}
  \city{Sunnyvale}
  \state{California}
  \country{USA}
  \postcode{94085}
}

%%
%% By default, the full list of authors will be used in the page
%% headers. Often, this list is too long, and will overlap
%% other information printed in the page headers. This command allows
%% the author to define a more concise list
%% of authors' names for this purpose.
\renewcommand{\shortauthors}{Zhu and Xiao, et al.}

%%
%% The abstract is a short summary of the work to be presented in the
%% article.
\begin{abstract}
As a special field in deep learning, Graph Neural Networks (GNNs) focus on extracting intrinsic network features and have drawn unprecedented popularity in both academia and industry. Most of the state-of-the-art GNN models offer expressive, robust, scalable and inductive solutions empowering social network recommender systems with rich network features that are computationally difficult to leverage with graph traversal based methods.

Most recent GNNs follow an encoder-decoder paradigm to encode high dimensional heterogeneous information from a subgraph\footnote{node, edges, subgraphs or entire graph; we focus on node embeddings in this work} onto one low dimensional embedding space. However, one single embedding space usually fails to capture all aspects of graph signals. In this work, we propose boosting-based meta learner for GNNs, which automatically learns multiple projections and the corresponding embedding spaces that captures different aspects of the graph signals. As a result, similarities between sub-graphs are quantified by embedding proximity on multiple embedding spaces. \adagnn\, performs exceptionally well for applications with rich and diverse node neighborhood information. Moreover, \adagnn\, is compatible with any inductive GNNs for both node-level and edge-level tasks.
\end{abstract}

% , especially when the information space is divisible
%%
%% Keywords. The author(s) should pick words that accurately describe
%% the work being presented. Separate the keywords with commas.

%%
%% This command processes the author and affiliation and title
%% information and builds the first part of the formatted document.
\maketitle

\section{Introduction} \label{introduction}

Graph representation learning has advanced greatly in the recent years and has drawn attention in both academia and industry because GNNs are expressive, flexible, robust and scalable. 

Many recent GNN models learn the projections from node neighborhoods using different sampling, aggregation and transformations~\cite{grle16, peas14, Kipf:2016tc, hamilton2017inductive, xu2018how,DBLP:conf/icml/YouYL19, rossi20}. GNNs have been adopted in various academic and industrial settings, such as link prediction~\cite{ying2018graph}, protein-protein interaction~\cite{shen2021npi}, community detection \cite{wupc19}, and recommender systems~\cite{wahz18,yihc18}. Furthermore, \citet{aligraph,yihc18,lews19,maym18} develop fault-tolerant and distributed systems to apply graph neural networks (GNNs) to large graphs.

These GNN models focus on learning a single encoder projecting graph substructures to a representation embedding. However, low-dimensional node embeddings sometimes fail to capture all the high-dimensional information about node neighborhoods. For example, in a social network, users may connect to distinct neighbors who share different common interests and thus the semantic meaning of edges varies~\cite{yang19}. Most prior works aim to capture the diverse signals of node neighborhoods by increasing the embedding dimension. However, it is challenging to encode the local multi-modal information to one embedding space in many cases. Multi-head attention bridges the gap to some extent by learning different attention heads for different neighbors~\cite{velickovic18}. One promising alternative is to learn multiple representations where each embedding captures a specific aspect of the rich information about node neighborhoods. By projecting to multiple low-dimensional node embedding spaces, we find it extremely promising when the inter-node affinity correlates with the inter-embedding similarity in one or more sub-spaces instead of the entire space.

%LinkedIn professional social networks hundreds of millions of users associated with rich contents and local neighborhood. % 
Network machine learning applications usually consider exceptionally heterogeneous feature from node, topological signals, various neighborhoods and communities where there exists an ensemble of latent semantics under the network features~\cite{yang19}.  Here we propose boosting-based GNNs, which automatically learn projections to multiple low-dimensional embedding spaces from the high-dimensional graph contents and determine the focus of each embedding space according to the node neighborhoods. 

\setlength\belowcaptionskip{-3ex}

\begin{figure*}
	\begin{subfigure}[t]{0.48\textwidth}
		\centering
		\includegraphics[scale=0.27]{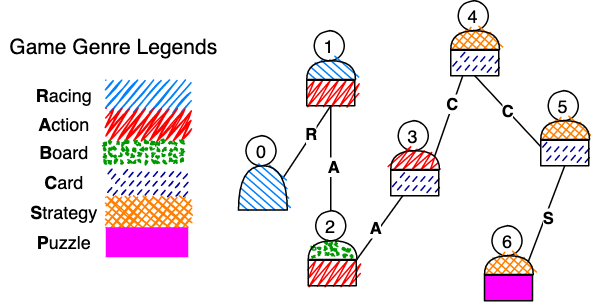}
		\caption{Most members have 2 different types of interests, and connections are annotated with reason for connection (common interests of game genres)}
		\label{fig:toy_example1}
	\end{subfigure}
	\begin{subfigure}[t]{0.48\textwidth}
		\centering
		\includegraphics[scale=0.24]{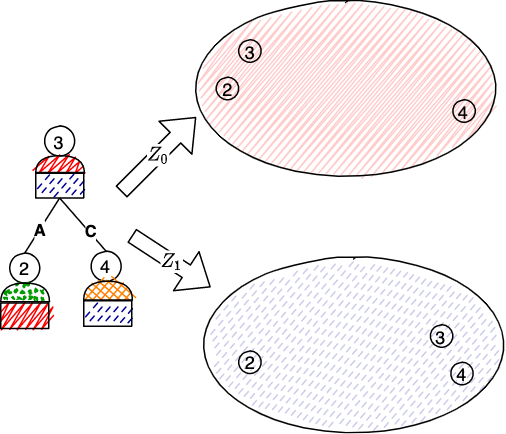}
		\caption{Multi-modal latent representation spaces}
		\label{fig:toy_example2}
	\end{subfigure}
	\caption{A toy example of a Twitch social network of 7 users and 6 links.}
	\label{fig:toy_example}
\end{figure*}

\subsection{Preliminaries}
\paragraStartHighlight{Static Graph} Despite that the concrete formulation of graph problems varies, we use a quadruplet 
\begin{equation} \label{eq:graph}
\mathcal{G} = (\mathcal{V}, \mathcal{E}, \{\mathbf{v}_i\in \mathbb{R}^{d_\mathcal{V}}\}, \{\mathbf{e}_{i,j}\in \mathbb{R}^{d_\mathcal{E}}\})
\end{equation}
to denote the holistic information about a graph, where $\mathcal{V}$ and $\mathcal{E}\subseteq\mathcal{V}\times\mathcal{V}$ are the sets of all nodes and edges respectively. They are endowed by node features $\mathbf{v}$ and edge features $\mathbf{e}$. 

\paragraStartHighlight{Dynamic Graph} Recent works investigated deeply into the representation learning in dynamic graphs and corresponding variations of GNNs~\cite{rossi20, kumar2019predicting} which consider both the evolution of $\mathcal{V}(t), \mathcal{E}(t)$ in terms of cardinality and feature updates $\mathbf{v}_i$ and $\mathbf{e}_{i,j}$. For simplicity, we assume $\mathcal{V}=\{v_1, v_2,\dots, v_n\}$ is the universe of all nodes that ever exist during the trajectory of the dynamic network; and features are timestamped $\mathbf{v}_i(t),\mathbf{e}_{i,j}(t)$.

\paragraStartHighlight{Embedding Space} An embedding space $\mathcal{Z}\subseteq\mathbb{R}^{d_\mathcal{Z}}$ is a vector space onto which we project nodes $\mathcal{V}$. Let $s_{i,j}=S_\mathcal{Z}(\mathbf{z}_i, \mathbf{z}_j)$ denote a similarity measure between two embeddings.

\paragraStartHighlight{Encoder-Decoder framework}  We let $\mathcal{G}_i$ denote the neighborhood\footnote{A $L$-hop neighborhood$\mathcal{N}_i^L$ considering all the nodes and edges within a radius of $L$-edges w.r.t $i$ is a widely used definition.}around node $v_i$, which comprises nodes, edges, node features and edge features. Conceptually an encoder $\enc: \mathcal{V} \rightarrow \mathbb{R}^{d_\mathcal{Z}}$ projects nodes to an embedding space whereas the actual model maps neighborhood $\mathcal{G}_i$  to an embedding $\mathbf{z}_i$. The concrete choices of decoders vary in different applications. For example, in node-level supervised learning with labels $Y_{\mathcal{V}}=\{\mathbf{y}_i\}_{v_i\in\mathcal{V}}$, our goal is to optimize encoder and decoder  $$\argmin_{\enc, \dec} \mathbb{E}_{v_i\in\mathcal{V}}[\mathcal{L}(\mathbf{y}_{i}, \dec(\enc(\mathcal{G}_i)))].$$ 
Whereas for pairwise utility prediction, we are usually accompanied with an inter-node utility labels $Y_{\mathcal{E}}=\{y_{i,j}\}_{(v_i,v_j)\in\mathcal{E}}$ and the goal reduces to searching the best encoder and decoder\footnote{Sometimes decoders do not contain trainable parameters (e.g. dot product, cosine similarity), then the optimization is only searching the encoder parameter space.} $$\argmin_{\enc,\dec} \mathbb{E}_{(v_i,v_j)\in \mathcal{E}}[\mathcal{L}(y_{i,j}, \dec(\enc(\mathcal{G}_i), \enc(\mathcal{G}_j)))].$$Although GNNs can also be applied to other tasks, in this paper, our discussion covers the above two categories of applications.

\subsection{Multi-modal embedding spaces}
Figure \ref{fig:toy_example1} illustrates a toy social network example where game players have preferences of genres of games and the edges represent the friendship. As shown in figure \ref{fig:toy_example1}, user 0 and 1 established a friendship connection through their common interests of racing game; while user 5 and 6 are friends because of strategy games. 
\begin{itemize}
	\item \textbf{Multiple embedding spaces} Our goal is to learn multiple encoders $\enc^k:\mathcal{V}\rightarrow\mathcal{Z}^k, k=1,\cdots ,K$ on $\mathcal{G}$ such that $\mathbf{z}^0_2 \approx \mathbf{z}^0_3$ in $\mathcal{Z}^0$ whereas  $\mathbf{z}^1_4 \approx \mathbf{z}^1_3$ in $\mathcal{Z}^1$ because they are close in the social network via different "semantics". As shown in figure~\ref{fig:toy_example2}, we seek to learn two projections such that their embedding similarities are guaranteed in different spaces. 
	\item \textbf{Node decoder} is used in node classification or recommendation tasks. The objective is to learn encoder (and decoder) for node-level labels. In the multi embedding setup, we consider the following form of node decoders:
	$$\dec_{node}: \{\bigcup_{k=1}^K\mathcal{Z}^k\}\rightarrow \mathbb{R}^{d_Y}$$
	\item \textbf{Inter-node proximity} is determined by multiple embedding spaces. In a supervised link prediction setup, the final prediction is based on decoders to "combine" the similarities between two nodes in $K$ multiple embedding spaces. $$\dec_{pairwise}: \{\bigcup_{k=1}^K\mathcal{Z}^k\}\times  \{\bigcup_{k=1}^K\mathcal{Z}^k\} \rightarrow \mathbb{R}$$
    \iffalse
	\item \textbf{User attributes} The node unobserved attributes are represented by multiple embeddings in different spaces . We can infer user unobserved attributes from its neighborhood attributes based on the network homophily nature~\cite{bramoulle2012homophily}. In Figure ~\ref{fig:toy_example2}, we can infer that user 3 might be interested in racing and action games based on its distances to user 2 and 4.
	\fi
\end{itemize}

Unfortunately, as mentioned in Pal et al.\cite{Pal20}, the complex and noisy nature of graph data renders such assumption we made in the toy example, the existence of explicit edge ontology, unrealistic.

\section{Related Work}

\paragraStartHighlight{Graph Embedding Models}, such as GCN (Kipf et al. \cite{kipf17}), GraphSAGE (Hamilton et al. \cite{hamilton18}), GAT (Velickovic et al. \cite{velickovic18}), can encode the subgraph $\mathcal{G}_i$ to a vector space $Z$. However, they only map the subgraph to one single embedding space instead of multiple embedding spaces.

\paragraStartHighlight{Multi-Embedding Models}, such as PinnerSage (Pal et al. \cite{Pal20}) and others (Weston et al. \cite{Weston13}), try to learn multi-modal embeddings via clustering method which is expensive. Moreover, it requires additional empirical inputs regarding number of clusters, similarity measure and pre-trained high quality embeddings which presumably capturing rich multi-modal signals.

\paragraStartHighlight{Temporal Network Embedding Models}, such as Jodie (Kumar et al. \cite{Kumar19}), TGAT (Xu et al. \cite{xu20}), TGN (Rossi et al. \cite{rossi20}), are designed for dynamic graphs, mapping from time domain to the continuous differentiable functional domain. Our boosting method can also be implemented onto dynamic graphs, collaborating with those temporal network embedding models.

\section{Present Work} 
In this work, we present an Ada-boosting based meta learner for GNNs (\adagnn) that is both model and task agnostic. \adagnn\ leverages a sequential boosting training paradigm that allows multiple different sub-learners to co-exist. Each embedding space ideally preserves unique inter-node similarity information. In this section, we mainly discuss the theoretical rationale behind our intuition regarding the advantages over single-embedding space using a node-level context as an example\footnote{In this work, we only experimented with node recommendation, link prediction and multi task learning}. Moreover, our approach works in a joint training fashion that diminishes the prerequisite of pre-trained embeddings.

%\paragraStartHighlight{Node Entropy} We assign probability values to each node as $p(v_i)=\frac{\exp(-\Delta_i)}{\sum_{j=1}^{|\mathcal{V}|}\exp(-\Delta_j)}$, where $\Delta_i$ denotes the difference between node feature vector $x_i$ and its first order neighbors' node features, $\Delta_i=\sum_{k=1}^K\Delta_i^k=\sum_{k=1}^K\sum_{j\in\mathcal{N}_i\cap\mathcal{E}^k}w^k_{ij}(x_j-x_i)^T(x_j-x_i)$. Then we derive the node entropy as $H(v_i)=-p(v_i)\log p(v_i)$.

\subsection{Problem Definition}
\iffalse
\begin{definition} \label{def:encode}
Given probabiltiy threshold $\tau$, embedding space $\mathcal{Z}$ fully encodes labels $Y$ iff.
$$
P(Y|\mathcal{Z})\geq\tau.
$$
\end{definition}
\paragraStartHighlight{Problem Formulation and existing GNNs} Given node $v_i$ and its neighborhood $\mathcal{G}_i$, the existing GNNs learn the encoder and the decoder such that the label of $v_i$ is best encoded by the embedding space onto which $\mathcal{G}_i$ is projected.

In the new framework, the AdaGNN learns multiple encoders that projects neighborhood $\mathcal{G}_i$ to multiple embedding spaces. The idea is that each embedding space fully encodes a subset of labels and together the embedding spaces capture the full label.
\fi

Taking static graph or dynamic graph $\mathcal{G}$ in Eq. \ref{eq:graph} as input, we project the neighborhood of each node $\mathcal{G}_i$ onto embedding space such that embedding space fully encodes the labels, as defined below.

\begin{definition} \label{def:encode}
Given probabiltiy threshold $\tau$, embedding space $\mathcal{Z}$ fully encodes labels $Y$ iff.
$$
P(Y|\mathcal{Z})\geq\tau.
$$
\end{definition}

Inspired by the toy example in Figure~\ref{fig:toy_example}, we assume that labels are affected by different aspects of neighborhoods. Then we can find an embedding space such that labels affected by one specific aspect is encoded in this embedding space while labels not affected by this aspect isn't encoded, as defined below.

\iffalse
We assume that there are $K$ possible latent signals, which directly affect the labels. In other words, node-level or edge-level labels are covered by $K$ subsets affected by different latent relations. In general, labels are defined as $Y$. $Y^k\subset Y$ are labels induced by $k$th latent relation, as defined below.
\fi

\begin{definition} \label{def:diffusion_label}
Diffusion induced labels $Y^k$ is a subset of labels $Y$ such that $Y^k$ can be decoded from one embedding space $\mathcal{Z}^k$ and $Y-Y^k$ can't be decoded from this embedding space.
\end{definition}

Embedding space $\mathcal{Z}^k$ is specifically related to the diffusion induced labels $Y^k$. Consequently we use $Y^k$ instead of $Y$ as labels in the training of encoder $\enc^k$, decoder $\dec^k$. 

We further assume that labels $Y$ are partitioned by $K$ diffusion induced labels $\{Y^k\}_{k=1}^K$. Then the graph learning problem is modified into multi-embedding learning as follows:

\paragraStartHighlight{Problem} Given a graph $\mathcal{G}$, learn a set of embedding spaces $\{\mathcal{Z}^k\}_{k=1}^K$ and corresponding encoders, decoders by reducing the discrepancy between decoder outputs and diffusion induced labels $\{Y^k\}_{k=1}^K$.

\subsection{Problem Context} \label{problem_context}
%\subsubsection{Multi-modal Embeddings}

In this section, node-level labels are used as example to justify the intuition behind multi-modal embedding spaces. Other cases are easy to replicate, such as using $P(y_{i,j}|\mathbf{z}_i,\mathbf{z}_j)$ in link prediction, where $y_{i,j}=\{0,1\}$ denotes the edge existence. We assume that there exists an ideal embedding space such that node embeddings encode all the latent signals:

\begin{definition} \label{def:embedding_space}
Embedding space $\mathcal{Z}^{ideal}$ is defined as \textbf{ideal} embedding space on the whole graph $\mathcal{G}$ iff. it fully encodes all labels: 
\begin{equation} \label{eq:ideal}
\forall_{v_i \in \mathcal{V}}\ P(\mathbf{y}_i|\mathbf{z}_i)\geq \tau,
\end{equation}
\end{definition}
where $\mathbf{y}_i\in Y$ represents node-level vector label, $\mathbf{z}_i$ is the node embedding of $v_i$ and $\tau\in[0,1]$ is a probability threshold.

Assume vector space $\mathcal{Z}^{ideal}$ is spanned by a set of orthonormal basis $\{\mathbf{b}_n\}_{n=1}^N$, then the node embeddings are linear combinations of these basis vectors with coefficients: $z_{i,n}=\mathbf{z}_i\cdot\mathbf{b}_n. $

In Eq. \ref{eq:ideal}, given node $v_i$, write node embedding in orthonormal basis and use Bayes' theorem:
\begin{equation} 
P(\mathbf{y}_i|\mathbf{z}_i) = P(\mathbf{y}_i|\sum_{n=1}^Nz_{i,n}\mathbf{b}_n) = \frac{P(\mathbf{y}_i, \sum_{n=1}^Nz_{i,n}\mathbf{b}_n)}{P(\sum_{n=1}^Nz_{i,n}\mathbf{b}_n)}.
\end{equation}

We further assume that only a small subset of basis vectors have correlation with label $\mathbf{y}_i$. 
The idea behind this assumption is that \textit{most of the volume of the high-dimensional cube is located in its corners.}
When $N\gg1$, most of the coefficients are close to zero after normalization.
%The idea behind this assumption is intuitive: there exist multiple independent latent relations among the graph and labels affected by each relation is encoded in subspace of the high dimensional vector space. 
Let $\{\mathbf{b}_{f(n)}\}_{n=1}^d$ be the set of basis vectors which have nontrivial correlation with label $y_i$, then $d\ll N$:
\begin{equation} \label{eq:subspace}
P(\mathbf{y}_i|\mathbf{z}_i) \sim \frac{P(\mathbf{y}_i, \sum_{n=1}^dz_{i,f(n)}\mathbf{b}_{f(n)})}{P(\sum_{n=1}^dz_{i,f(n)}\mathbf{b}_{f(n)})}=P(\mathbf{y}_i|\sum_{n=1}^dz_{i,f(n)}\mathbf{b}_{f(n)}),
\end{equation}
and $\{\mathbf{b}_{f(n)}\}_{n=1}^d$ spans a vector space $\mathcal{Z}^0\subset\mathcal{Z}^{ideal}$. For nodes satisfying Eq. \ref{eq:subspace}, they are encoded in $\mathcal{Z}^0$ and we define them as $Y^0$. For nodes not satisfying Eq. \ref{eq:subspace}, a new vector space $\mathcal{Z}^1$ and set $Y^1$ can be obtained using the above procedure. Note that the intersection between $Y^0$ and $Y^1$ is not necessarily empty. Repeat until all nodes are encoded. A set of vector space are obtained:
\begin{align}
Y&=\cup_{k=1}^KY^k,\\
\mathcal{Z}^{ideal}&=\mathcal{Z}^1\oplus \mathcal{Z}^2\oplus \mathcal{Z}^3\oplus\cdots\oplus \mathcal{Z}^K,
\end{align}
and 
\begin{equation}
\forall_{v_i \in \mathcal{V}}\ \exists k, \text{ s.t. } \mathbf{y}_i\in Y^k,\ P(\mathbf{y}_i|\mathbf{z}^k_i)\geq\tau,
\end{equation}
where $\mathbf{z}^k_i$ is the node embedding in embedding space $\mathcal{Z}^k$.

In existing works, encoders and decoders parameterize $P_{\enc, \theta}(\mathbf{z}_i|\mathcal{G}_i)$ and $P_{\dec, \theta}(\mathbf{y}_i|\mathbf{z}_i)$ using neural networks. In this paper we introduce a new approach: modeling $P_{\enc^k, \theta}(\mathbf{z}^k_i|\mathcal{G}_i)$ and $P_{\dec^k, \theta}(\mathbf{y}_i|\mathbf{z}^k_i)$. When $d\ll N$, the new approach obtains a large advantage and a new embedding space can always be found to increase the overall accuracy before the ideal embedding space is achieved. Moreover, increasing the number of embedding spaces can bound the generalization gap, which is discussed in Section \ref{margin_theory}.

\subsection{Multiple Embeddings Generation}
Inspired by the procedure above of finding new embedding space and the weight-updating ability of boosting, we adopt AdaBoosting as the meta learner with homogeneous GNNs as weak learners to encode underlying relations from weighted labels, which circumvents the difficulties in clustering method\cite{Pal20}. Each learner has a GNN encoder $\enc^k$ that projects $\mathcal{G}_i$ onto one embedding space $\mathcal{Z}^k$. From a high-level perspective, we learn such embedding projections  in an iterative fashion: each training data point is associated with a weight based on the previous weak learner's (GNN) error and this weight encourages the next weak learner to focus on the data points with misclassified labels from the previous learner.
%\begin{proposition}
%Find embedding space $\mathcal{Z}_k$ and the corresponding encoder $\text{ENC}_k:\mathcal{V}\rightarrow\mathcal{Z}_k$ iteratively on diffusion induced network $\mathcal{G}^k$ computed by previous embedding space $\mathcal{Z}_{k-1}$.
%\end{proposition}

\iffalse
\begin{algorithm}
\SetAlgoLined
\begin{flushleft}
 1. initialize the edge weights $w^{0}_{i,j}=\frac{1}{n}$, $n$ is the number of edges.
 
 2. \For{$k=1$ to $K$}{

  (a) Construct diffusion induced labels $Y^k$ with weights

  (b) Fit a GNN encoder $z_i=h^k(\mathcal{G}_i)$ and decoder $s^k_{i,j}=d(z_i,z_j)$ on $\mathcal{G}^k$.
  
  (c) Update the edge weights with learning rate $\alpha$:
  $$
  w_{i,j}^{k}=w_{i,j}^{k-1}\cdot\exp({-\frac{1}{2}*\alpha *y_{i,j}^T\log s^k_{i,j}}),\ \ i=1,\cdots,n.
  $$
  
  (d) Re-normalize $w_{i,j}^{k}$.
 }
\end{flushleft}
\caption{AdaGNN with SAMME.R}
\label{alg:SAMME.R}
\end{algorithm}

For each learner and associated embedding space, edges are called misclassified if the learner fails to predict the existence of them. In the other word, for given metric $d$ and threshold $t>0$, misclassified edge $(v_i, v_j)$ has $s_{ijk} < t$. If there exists misclassified edges, it means that the similarity of those misclassified edges isn't preserved in this embedding space, so we could get a new embedding space capturing this similarity from misclassified edges.
\fi

\begin{lemma}\label{lemma:new_embedding1}
Given a learner with corresponding embedding space $\mathcal{Z}^k$, there exists a new embedding space $\mathcal{Z}^{k+1}\subset\mathcal{Z}^{ideal}$ capturing different information from the existing embedding space until the number of misclassified labels is zero.

\end{lemma}

\begin{proof}
For this learner and embedding space $\mathcal{Z}^{k}$, if the number of misclassified labels is nonzero, there exists at least one misclasssified label. The relation in this label is not preserved, so there exists a new embedding space $\mathcal{Z}^k\subset\mathcal{Z}^{ideal}$ preserving this relation using procedure in Section \ref{problem_context}. Note that new embedding space may not be able to encode all the classified labels in previous learner.
\end{proof}

The "artificial" diffusion induced labels are created by boosting weights. For link prediction, if the number of misclassified edges is nonzero, i.e., edges $\mathcal{E}^{\text{mis}}\neq\emptyset$. We can generate new diffusion induced labels only including misclassified edges by setting the weights of classified edges to zero. In experiments, we increase the weights of misclassified edges to some extent until at least one previously misclassified edge is classified correctly in next weak learner instead. If it can't be achieved, boosting will stop. 

%Given an edge $(v_i, v_j)$, if the weight of this edge $w_{ij}$ is large enough, the learner can find an embedding space $k$ such that $s_{ijk}\geq t$.

%In order to find the above embedding space $\mathcal{Z}_{k+1}$, we change our binary cross entropy loss function into
%$$
%L=-\sum_{v_i,v_j\in\mathcal{V}}w_{ij}e_{ij}\log s_{ij} + w_{ij}(1-e_{ij})\log (1-s_{ij}),
%$$
%where $w_{ij}$ is the weight of edge $e_{ij}\in\mathcal{E}$ and $\sum_{v_i,v_j\in\mathcal{V}}w_{ij}=1$.

%Given an edge $(v_i, v_j)$, if the weight of this edge $w_{ij}$ is large enough $\frac{w_{ij}}{\text{other weights}}\gg1$, the learner could find an embedding space $k$ such that $s_{ijk}\leq t$.

\begin{lemma}\label{lemma:new_embedding2}
Embedding space $\mathcal{Z}^k, \mathcal{Z}^{k+1}$ of current and next learner constructed as above capture different information .
\end{lemma}

\begin{proof}
For current learner, if its training error is nonzero (otherwise we stop boosting), there exists at least one misclassified label. We then increase the weight of this data point for the next learner so that the label will be classified correctly. $\mathcal{Z}^{k+1}$ preserves the latent relation that affects this label, while $\mathcal{Z}^{k}$ does not. %Then $\mathcal{Z}_{k}$ and $\mathcal{Z}_{k+1}$ capture different information.
%Then in the embedding space $\mathcal{Z}_{k+1}$ of learner $h_{k+1}$, the distance between $v_i$ and $v_j$ is $s_{ij,k+1}\geq t$.
\end{proof}

%The question, then, is what is the best weight updating rule. To answer this question, we need to define what we mean by best.

Therefore, there will always exist a new embedding space capturing different information from the original embedding space until the number of misclassified labels is zero and the embedding spaces can be found by increasing the weights of misclassified labels.

It remains unresolved about what the optimal weight updating rule is, and it depends on the definition of optimality. The most common definition is to achieve the lowest misclassification error rate. Usually it is assumed that the training data are i.i.d. samples from an unknown probability distribution. Then we can derive the boosting algorithm based on different misclassification error rate. In Zhu et al.\cite{Zhu06}, they proposed two algorithms: SAMME and SAMME.R, and proved that the two algorithms minimize the misclassification error rate for discrete predictions and real-valued confidence-rated predictions respectively. More generally, one can use gradient boosting, which is left for future work.

\begin{lemma}\label{lemma:adaboost}
For AdaBoost, adding a new learner with weights $w^{k}$ as in Algorithm 1 will minimize the misclassification error.
\end{lemma}

\begin{proof}
Zhu et al.\cite{Zhu06} proved this for SAMME and SAMME.R algorithm using a novel multi-class exponential loss function and forward stage-wise additive modeling.
\end{proof}

\begin{theorem}
$\{\mathcal{Z}^k\}_{k=1}^K$ in problem definition can be found using boosting and it minimizes the misclassification error.
\end{theorem}
\begin{proof}
It's straightforward by Lemma ~\ref{lemma:new_embedding1}, ~\ref{lemma:new_embedding2} and ~\ref{lemma:adaboost}.
\end{proof}

\subsection{Graph Neural Network}
GNNs use the graph structure, node features and edge features to learn a node embedding $\mathbf{z}_i\in\mathcal{Z}$ for a node $v_i\in\mathcal{V}$. Modern GNNs follow a neighborhood aggregation strategy, where we iteratively update the representation of a node by aggregating representations
of its neighbors. After $L$ iterations of aggregation, a node’s representation captures the structural
information within its $L$-hop network neighborhood. Formally, the $L$-th layer of a GNN is
\begin{equation*}
\begin{split}
\mathbf{a}_i^{(l)}&=\agg(\{\mathbf{z}_j^{(l-1)},\ v_j\in\mathcal{N}(v_i)\}),\\
\mathbf{z}_i^{(l)}&=\combine(\mathbf{z}_i^{(l-1)},\ \mathbf{a}_i^{(l)}),
\end{split}
\end{equation*}
where $\mathbf{z}_i^{(l)}$ is the node embedding of $\mathbf{v}_i$ with the message of $l$-hop network neighbors,  $\mathbf{a}_i^{(l)}$ is the message aggregated from $l$-hop neighbors and we initialize $\mathbf{z}_i^{(0)}=\mathbf{v}_i$. 
In this work, attention mechanism or mean pool is used to aggregate the neighbor representations. Aggregated messages are further combined with self embeddings using fully connected layers.
We will drop superscripts and use $\mathbf{z}_i$ to represent node embeddings in the following discussion.

\iffalse
\begin{figure}[H]
    \begin{subfigure}[t]{0.235\textwidth}
	\centering
	\includegraphics[scale=0.235]{figures/illustration1.png}
	\caption{Neighborhood affected by $\xi^k$}
	\label{fig:illustration1}
	\end{subfigure}
	\begin{subfigure}[t]{0.235\textwidth}
	\centering
	\includegraphics[scale=0.235]{figures/illustration2.png}
	\caption{Multi-task learning}
	\label{fig:illustration2}
    \end{subfigure}
    \caption{Visualization of multi-task learning. $z_i$ represents node embedding; $s_{i,j}$ is a scalar between $[0,1]$, showing the probability of the edge existence; $r_i$ is a vector showing the games user is interested in.}
    \label{fig:architecture}
\end{figure}
\fi

Following encoder-decoder framework, node neighborhood $\mathcal{G}_i$ is projected to a vector space $\mathcal{Z}$ which is then decoded for different tasks: a) \emph{Link Prediction}, where edge existence is predicted using the similarity between two node embeddings; b) \emph{Node Recommendation}, where the node-level labels are predicted only using the information of neighbors excluding themselves; c) \emph{Multi-Task Learning}, where link prediction and node recommendation are trained together using the same encoder. %Furthermore, one specical case of link prediction is future link prediction, where we predict the edge exsitence in the future using the past graph structure.

For link prediction task, we have
\begin{equation*}
\begin{split}
s_{i,j}&=\dec_{pairwise}(\mathbf{z}_i,\ \mathbf{z}_j),\\
L&=-\sum_{(v_i,v_j)\in\mathcal{E}}w_{i,j}y_{i,j}\log s_{i,j} + w_{i,j}(1-y_{i,j})\log (1-s_{i,j}),
\end{split}
\end{equation*}
where $s_{i,j}$ is the similarity between two node embeddings $\mathbf{z}_i,\mathbf{z}_j$, label $y_{i,j}\in Y_{\mathcal{E}}$ represents the edge existence and $L$ is the binary cross entropy loss function with corresponding edge weights $w_{i,j}$. Negative sampling is included.

For node recommendation task, we have
\begin{equation*}
\begin{split}
\mathbf{r}_i&=\dec_{node}(\mathbf{z}_i),\\
L&=-\sum_{v_i\in\mathcal{V}}w_i\mathbf{y}_i^T\log \mathbf{r}_i,
\end{split}
\end{equation*}
where $\mathbf{y}_i\in Y_{\mathcal{V}}$ is the vector label and $w_i$ are node weights. Note that $w_{i,j}$ and $w_i$ refer to weights in different tasks and thus are not related.
 
%For multi-task learning, we learn the link prediction task and node recommendation task at the same time using one uniform encoder and two different decoders, as in Figure \ref{fig:architecture}. Specifically, we add the loss functions of link prediction and node recommendation together with weights such that the the number of total nodes are comparable in two tasks.

We next introduce \adagnn, using link prediction task as an example. \adagnn\ includes a series of GNNs, called weak learners.  Each weak learner is trained on weighted sample points and the weights used in present learner depend on the training errors from the previous learner.

%We only boost the weights of edges in link prediction task. 
Mathematically, the $k$th weak learner can be written as 
\begin{equation*}
\begin{split}
\mathbf{z}^k_i=\enc^k(\mathcal{G}_i):&\quad \mathcal{G}\to \mathcal{Z}^k,\\
s_{i,j}^k=\dec^k(\mathbf{z}^k_i,\mathbf{z}^k_j): &\quad \mathcal{Z}^k\times\mathcal{Z}^k\to [0,1].
\end{split}
\end{equation*}
%with threshold $t$ such that
%\begin{align*}
%&s_{ijk}\geq t\to 1,\\
%&s_{ijk}<t\to 0.
%\end{align*}

Given $K$ weak learners, the similarity between two nodes $v_i,v_j\in\mathcal{V}$ for the metalearner is 
$$
s_{i,j}=\sum_{k=1}^K\alpha^ks_{i,j}^k=\sum_{k=1}^K\alpha^k \dec^k(\enc^k(\mathcal{G}_i), \enc^k(\mathcal{G}_j)),
$$
where $\alpha^k>0$, $\sum_{k=1}^K\alpha^k=1$.

\subsection{Generalization Error of AdaGNN} \label{margin_theory}
The calculation of generalization error for \adagnn\ is mainly based on the generalization error of boosting from Schapire et al.\cite{Schapire98} and the VC dimension of GNN from Scarselli et al.\cite{Scarselli18}. We start with the definitions of the functional space of boosting GNNs.

\iffalse
\begin{lemma}
There exists $g(\mathcal{G}_i,\mathcal{G}_j)=\sum_{k=1}^K\frac{1}{K}d(h^k(\mathcal{G}_i), h^k(\mathcal{G}_j))$ with $K$ independent functions $h^k$, such that $g\in\mathcal{C}_K$ and
\begin{equation*}
E_{g\sim\mathcal{Q}}[g(\mathcal{G}_i,\mathcal{G}_j)]=f(\mathcal{G}_i,\mathcal{G}_j).
\end{equation*}
\end{lemma}
\fi

\begin{definition}
Let $\mathcal{H}$ be the functional space of GNN encoders from graph $\mathcal{G}$ to embedding space $\mathcal{Z}$, $h_k\in\mathcal{H}$ be the encoders and $d$ be the uniform decoder for all encoders, then we define $\mathcal{C}$ be the set of a weighted average of weak learners from $\mathcal{H}$:
\begin{equation*}
\mathcal{C}\doteq\{f: \mathcal{G}\to\sum_{h^k\in\mathcal{H}}\alpha^k s_{i,j}^k\ |\ \alpha^k\geq0; \sum_{k=1}^K\alpha^k=1 \},
\end{equation*}

and $\mathcal{C}_K$ be the set of unweighted averages over $K$ weak learners from $\mathcal{H}$:
\begin{equation*}
\mathcal{C}_K\doteq\{f: \mathcal{G}\to\frac{1}{K}\sum_{k=1}^Ks_{i,j}^k\ |\ h^k\in\mathcal{H} \},
\end{equation*}
where $s_{i,j}^k= d(h^k(\mathcal{G}_i), h^k(\mathcal{G}_j))$.

\end{definition}

Any projection $f\in\mathcal{C}$ can be associated with a distribution over $\mathcal{H}$ defined by the coefficients $\alpha^k$.
In other word, any such $f$ defines a natural distribution over $\mathcal{H}$ where we draw function $h^k$ with probability $\alpha^k$. By choosing $K$ elements of $\mathcal{H}$ independently at random according to this distribution and take an unweighted sum, we can generate an element of $\mathcal{C}_K$. Under such construction, we map each $f\in\mathcal{C}$ to a distribution $\mathcal{Q}$ over $\mathcal{C}_K$. That is, a function $g\in\mathcal{C}_K$ distributed according to $\mathcal{Q}$ can be sampled by choosing $h^1,\cdots,h^K$ independently at random according to the coefficients $\alpha^k$ and then defining $g(\mathcal{G}_i,\mathcal{G}_j)=\sum_{k=1}^K\frac{1}{K}d(h^k(\mathcal{G}_i), h^k(\mathcal{G}_j))$.

The key property about the relationship between $f$ and $Q$ is that each is completely determined by the other. Obviously $Q$ is determined by $f$ because we defined it that way, but $f$ is also completely determined by $Q$ as follows:
\begin{equation}
f(\mathcal{G}_i,\mathcal{G}_j)=E_{g\sim\mathcal{Q}}[g(\mathcal{G}_i,\mathcal{G}_j)].
\end{equation}

\begin{theorem}
Let $\mathcal{D}$ be a distribution over $\mathcal{G}\times\mathcal{G}\times\{0, 1\}$, and let $S$ be a sample of $m$ node pairs chosen independently at random according to $\mathcal{D}$. Suppose the base-classifier space has VC-dimension $d$, and let $\delta>0$. Assume that $m\geq d\geq 1$. Then with probability at least $1-\delta$ over the random choice of the training set $S$, every weighted average projection $f\in\mathcal{C}$  satisfies the following bound for all $\theta>0$:
% \begin{equation*}
% \begin{split}
% &P_{\mathcal{D}}[e_{ij}(t-f(\mathcal{G}_i^L,\mathcal{G}_j^L))\leq0]\leq P_S[e_{ij}(t-f(\mathcal{G}_i^L,\mathcal{G}_j^L))\leq\theta] \\
% &\qquad\qquad+ O(\frac{1}{\sqrt{m}}(\frac{\log m\log|\mathcal{W}|}{\theta^2}+\log(\frac1\delta))^{1/2}),
% \end{split}
% \end{equation*}
\begin{equation} \label{eq:bound}
\begin{split} 
&P_{\mathcal{D}}[(y_{i.j}-\tau)(f(\mathcal{G}_i,\mathcal{G}_j)-\tau)\leq0]\\
&\leq P_S[(y_{i,j}-\tau)(f(\mathcal{G}_i,\mathcal{G}_j)-\tau)\leq\theta] \\
&+O(\frac{1}{\sqrt{m}}(\frac{d\log^2(\frac{m}{d})}{\theta^2}+\log(\frac1\delta))^{1/2}),
\end{split}
\end{equation}

\begin{proof}
	Using function $g$ constructed as above, this proof follows the same idea in Schapire et al.\cite{Schapire98}.
\end{proof}
\end{theorem}
\vspace{-1em}
We have proved that the generalization error bound of \adagnn\ depends on the number of training data and the VC-dimension of GNN. For one-layer GNN such as GCN and GraphSage, the VC-dimension has been calculated by Scarselli et al\cite{Scarselli18}. The generalization error is plotted and discussed in Section 4.3.3.

\iffalse
\subsection{Expressivity of AdaGNN}
To do: WL-test
\fi

\subsection{Training}
A single iteration of training in AdaGNN involves two parts: a) training a weak learner; b) boosting the label weights. For weak learners, we use different types of GNNs, including GraphSage, GAT for static graph and TGN for dynamic graph. For boosting, we use two different algorithms: a) \emph{SAMME.R (R for Real)} algorithm \cite{Zhu06}, it uses weighted class probability estimates rather than hard classifications in the weight-updating and prediction combination, which leads to a better generalization and faster convergence. b) \emph{AdaBoost.R2} algorithm \cite{Drucker97}, it uses bootstrapping, making it less prone to overfitting. The boosting algorithm using SAMME.R for link prediction is shown in Algorithm 1. AdaBoost.R2 or node-level tasks are similar.
%. c) \emph{AdaBoost.RT} algorithm \cite{solomatine04}, it's specific for regression problems.

\begin{algorithm}
\SetAlgoLined
\begin{flushleft}
 1. initialize the weights $w^{0}_{i,j}=\frac{1}{n}$, $n$ is the number of edges.
 
 2. \For{$k=1$ to $K$}{

  (a) Assign weights $w_{i,j}^k$ to labels $Y^k$.

  (b) Fit a GNN encoder $\mathbf{z}^k_i=\enc^k(\mathcal{G}_i)$ and decoder $s^k_{i,j}=\dec(\mathbf{z}^k_i,\mathbf{z}^k_j)$ on $Y^k$.
  
  (c) Update the edge weights with learning rate $\alpha$:
  $$
  w_{i,j}^{k}=w_{i,j}^{k-1}\cdot\exp({-\frac{1}{2}\cdot\alpha\cdot y_{i,j}^T\log s^k_{i,j}}).
  $$
  
  (d) Re-normalize $w_{i,j}^{k}$.
 }
\end{flushleft}
\caption{AdaGNN with SAMME.R}
\label{alg:SAMME}
\end{algorithm}

\iffalse
\begin{algorithm}
\SetAlgoLined
\begin{flushleft}
 1. initialize the observation weights $w_i=\frac{1}{n}$, $i=1,2,\cdots,n$.
 
 2. \For{$k=1$ to $K$}{
  (a) Fit a GNN $h_{k}(x)$ to the training data using sample weights $w_i^{(mk-1)}$.
  
  (b) Obtain the weighted class probability estimates 
  
  $$
  p_j^{(k)}(x)=\text{Prob}_w(c=j|x),\ \ j=1,\cdots,C.
  $$
  
  (c) Set 
  $$
  w_i^{(k)}=w_i^{(k-1)}\cdot\exp{-\frac{C-1}{C}y_i^T\log p^{(k)}(x_i)},\ \ i=1,\cdots,n.
  $$
  
  (d) Re-normalize $w_i^{(k)}$.
 }
\end{flushleft}
\caption{SAMME.R}
\label{alg:SAMME.R}
\end{algorithm}

\begin{algorithm}
\SetAlgoLined
\begin{flushleft}
 1. initialize the observation weights $w_i=\frac{1}{n}$, $i=1,2,\cdots,n$.
 
 2. \For{$k=1$ to $K$}{
  (a) Fit a GNN $h_{k}(x)$ to the bootstrapped training data using sample weights $w_i^{(k-1)}$.
  
  (b) Obtain the weighted class probability estimates 
  
  $$
  p_j^{(k)}(x)=\text{Prob}_w(c=j|x),\ \ j=1,\cdots,C.
  $$
  
  (c) Compute the estimator error and coefficient: $\text{error}=\sum_{i=1}^n w_i^{(k-1)}(1-p_{y_i}^{(k)}(x_i)),\ \  \beta=\frac{\text{error}}{1-\text{error}}$.
  
  (d) Set
  
  $$
  w_i^{(k)}=w_i^{(k-1)}\cdot\beta^{p^{(k)}(x_i)},\ \ i=1,\cdots,n.
  $$
  
  (e) Re-normalize $w_i^{(k)}$.
 }
\end{flushleft}
\caption{AdaBoost.R2}
\label{alg:AdaBoost.R2}
\end{algorithm} 

\fi

We focus on the following variations of \adagnn\ that are based on different combinations of underlying model, boosting algorithm and decoder types:

\begin{itemize}
\item \textbf{AdaGNN} AdaBoost-based GNN, such as AdaSage, AdaGAT, using SAMME.R algorithm.
\item \textbf{AdaGNN-b} AdaBoost-based GNN with bootstrapped training data, using AdaBoost.R2 algorithm.
\item \textbf{AdaGNN-nn} AdaBoost-based GNN with uniform non-linear decoder. We concatenate the node embeddings from all embedding spaces to be the new node embeddings and feed this new embeddings to uniform decoder. 
\end{itemize}

\section{Experiments}
We test \adagnn\ on three tasks and four real world social networks from diverse background applications.

\begin{itemize}
	\item \textbf{Twitch social networks}\cite{rozemberczki19} User-user networks where nodes correspond to Twitch users and edges to mutual friendships. Node features are games liked. The associated tasks are link prediction of whether two users have mutual friendships, node recommendation that recommends games each user like and multi-task learning.
	\item \textbf{Wikipedia}\cite{Kumar19} User-page networks where nodes correspond to Wikipedia users, Wikipedia pages and edges to one user editing one page. The associated tasks are future link prediction of whether one user will edit one page in the future.
	\item \textbf{Movielens}\cite{harper15}  User-movie networks where nodes correspond to users, movies and edges to one user rating or tagging one movie. The associated tasks are future link prediction of whether one user will rate or tag one movie in the 
	future.
	\item \textbf{Linkedin}  User-user networks where nodes correspond to Linkedin users and edges to mutual friendships. The associated task is link prediction of whether two users have mutual friendships. Future link prediction task is left for future work.
\end{itemize}
We also consider both transductive and inductive tasks w.r.t whether nodes are observed in training dataset. However it is worth to note that both our baseline models and \adagnn\, variations are inductive in nature.
\iffalse
\begin{itemize}
	\item \textbf{Transductive task} examines embeddings of the nodes that have been observed in training. %Specifically in the future link prediction task, to avoid violating temporal constraints, we predict the links that strictly take place posterior to all observations in the training data.
	
	\item \textbf{Inductive task} examines the inductive learning capability using the inferred representations of unseen nodes.
\end{itemize}
\fi
 For baseline models, we consider a few popular  and representative state-of-the-art models for static graph, GraphSage(\cite{hamilton18}), GAT(\cite{velickovic18}); as well as state-of-the art model for dynamic graph, TGN(\cite{rossi20}). By comparing the performance with baselines in \ref{performance} and experimentally, \adagnn\, shows advantages as follows:
\begin{itemize}
\item \adagnn\ outperforms baselines on all datasets for all tasks, especially when the information of neighborhoods is rich.
\item Multiple embedding spaces can capture different information, outperform single embedding space with higher dimension.
\item \adagnn\ is robust to the number of training data.
\end{itemize}

\begin{figure*}
    \centering
    \includegraphics[scale=0.42]{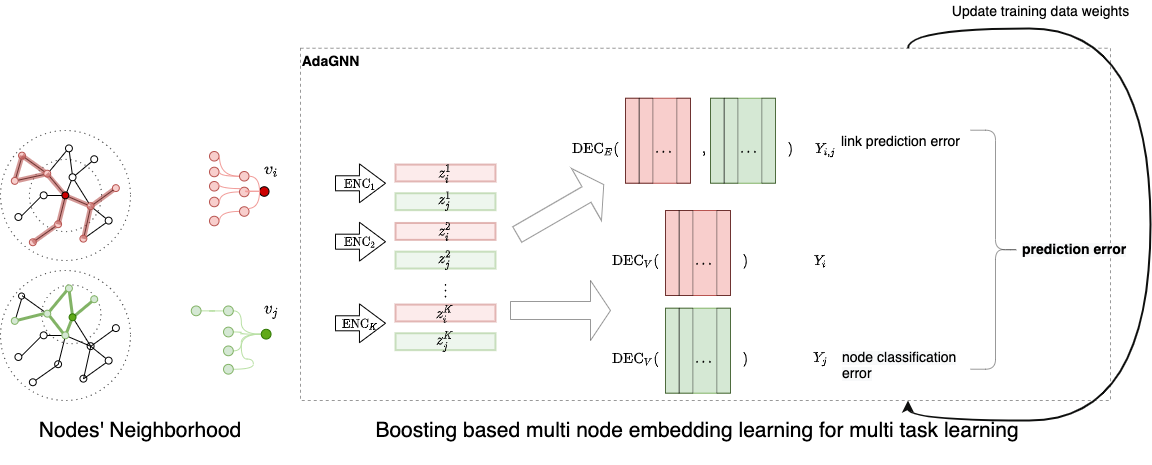}
    \caption{Multi-task learning. One uniform encoders and two different decoders are used.}
    \label{fig:architecture}
\end{figure*}

\subsection{Experimental Settings}
In addition to test link prediction and node classification separately, we also experiment on multi task learning  (as shown in Figure~\ref{fig:architecture}). Each training data point's weight will be updated based on the final combined predicting error, which makes the errors of two tasks comparable and gives high-degree nodes that are susceptible to high errors relatively higher weights in node recommendation tasks. The scalability and time complexity of \adagnn\, is highly coupled with the underlying model. It is obvious that \adagnn\, linearly growth with number of weak learners w.r.t time complexity. We also find that $5-10$ weak learners suffice for all datasets.

\paragraStartHighlight{Hyperparameters} For training of GNNs, including GraphSage, GAT, TGN, we use the Adam optimizer with a batch size of $200$. The learning rate is selected in $\{0.0001,\ 0.0005,\ 0.001\}$, the number of heads is selected in $\{1,\ 2,\ 3\}$, the number of neighbors is selected in $\{10,\ 20,\ 30\}$ and the number of layers is fixed as $1$. For boosting, the learning rate is selected in $\{1.0,\ 2.0,\ 3.0\}$. The number of negative samples is equal to positive samples except Linkedin datasets, where we have more positive samples. We do a simple grid hyper-parameter tuning for both \adagnn\, and baseline models, the best performance metrics are reported in the next section.

%  Jodie(\cite{Kumar19}), TGAT(\cite{xu20}),

\subsection{Performance Comparison with Baselines} \label{performance}
%Table \ref{table:multi_task} presents the results on multi-task learning, including link prediction and game recommendation. Boosting-based model clearly outperforms the baselines by a large margin in both transductive and inductive settings on all datasets, especially for the recommendation task. 

\setlength\belowcaptionskip{0ex}

\begin{table}[h]\small
\begin{center}
\resizebox{3.5in}{!}{%
\begin{tabular}{ |p{1.5cm}|p{1.25cm}|p{1.25cm}|p{1.25cm}|p{1.25cm}|p{1.25cm}|p{1.25cm}| } 
\hline
&\multicolumn{2}{|c|}{Wikipedia}&\multicolumn{2}{|c|}{Movielens}&\multicolumn{1}{|c|}{Linkedin} \\
\hline
\multicolumn{6}{|c|}{Static graph method}\\
\hline
GraphSage    & $93.35\pm0.2$ & $92.35\pm0.3$ & $72.28\pm0.2$ & $71.16\pm0.2$ &$60.65\pm0.6$\\
GAT                                & $94.72\pm0.2$ & $93.71\pm0.2$ & $75.69\pm0.3$ & $74.57\pm0.3$ &$69.50\pm1.0$\\
\hline
AdaSage & $94.51\pm0.1$ & $93.24\pm0.1$ & $74.47\pm0.3$ & $72.67\pm0.3$ &$68.45\pm0.1$\\
AdaGAT-b  & $95.83\pm0.1$ & $94.57\pm0.1$ & $79.71\pm0.3$ & $78.09\pm0.3$ &$\mathbf{75.98\pm0.3}$\\
AdaGAT            & $95.64\pm0.1$ & $94.43\pm0.1$ & $78.09\pm0.3$ & $76.45\pm0.3$ &$\mathbf{76.17\pm0.3}^{\ddagger}$\\
AdaGAT-n          & $\mathbf{96.32\pm0.1}$ & $\mathbf{95.86\pm0.1}$ & $\mathbf{82.86\pm0.2}^{\dagger}$ & $\mathbf{82.83\pm0.2}^{\dagger}$  &$\mathbf{78.83\pm0.1}^{\dagger}$\\
\hline
\multicolumn{6}{|c|}{Dynamic graph method}\\
\hline
% Jodie & ./94.62 & ./93.11 & ./. & ./. &\\
% TGAT  & ./95.34 & ./93.99 & ./. & ./. &\\
TGN & $\mathbf{98.46\pm0.1}^{\ddagger}$ & $\mathbf{97.81\pm0.1}^{\ddagger}$ & $\mathbf{81.34\pm0.4}$ & $\mathbf{80.35\pm0.4}$ &/ \\
\hline
AdaTGN& $\mathbf{99.00\pm0.1}^{\dagger}$ & $\mathbf{98.40\pm0.1}^{\dagger}$ & $\mathbf{82.66\pm0.4}^{\ddagger}$ & $\mathbf{81.45\pm0.4}^{\ddagger}$ & / \\
\hline
\end{tabular}}
\caption{Average Precision \% for Future Link Prediction Task. \textbf{First}$^{\dagger}$, \textbf{Second}$^{\ddagger}$, \textbf{Third} best performing method.}
\label{table:temporal_graph}
\end{center}

\begin{center}
\resizebox{3.5in}{!}{%
\begin{tabular}{ |c|p{1.25cm}|p{1.25cm}|p{1.25cm}|p{1.25cm}| } 
\hline
&\multicolumn{4}{|c|}{Twitch} \\
\hline
Task & \multicolumn{2}{|c|}{Link Prediction} & \multicolumn{2}{|c|}{Recommendation} \\
\hline
GraphSage         & $81.56\pm0.6$ & $73.55\pm0.5$ & $51.81\pm1.7$ & $47.45\pm0.6$\\
GAT-200             & $82.74\pm0.5$ & $74.35\pm1.0$ & $69.57\pm2.5$ & $62.24\pm1.5$\\
GAT-1024           & $83.24\pm0.8$ & $73.67\pm1.0$ & $70.45\pm1.5$ & $62.40\pm2.0$\\
GAT-3170           & $83.14\pm0.5$ & $73.57\pm0.3$ & $71.54\pm3.5$ & $63.41\pm1.5$\\
\hline
AdaSage               & $86.20\pm0.2$ & $77.56\pm0.5$ & $60.32\pm1.2$ & $55.54\pm0.6$\\
AdaGAT-200         & $85.73\pm0.1$ & $76.28\pm0.5$ & $77.82\pm0.6$ & $69.68\pm0.3$\\
AdaGAT-1024       & $86.65\pm0.1$ & $77.12\pm0.1$ & $81.67\pm0.3$ & $72.50\pm0.2$\\
AdaGAT-3170       & $85.92\pm0.3$ & $75.84\pm0.5$ & $80.18\pm2.0$ & $71.25\pm1.0$\\
AdaGAT-b             & $85.39\pm0.2$ & $76.35\pm0.5$ & $76.15\pm1.4$ & $68.70\pm0.8$\\
AdaGAT-n-200      & $\mathbf{86.96\pm0.2}^{\ddagger}$ & $\mathbf{77.37\pm0.1}^{\ddagger}$ & $\mathbf{95.91\pm0.1}^{\ddagger}$ & $\mathbf{78.06\pm0.3}$\\
AdaGAT-n-1024    & $\mathbf{87.34\pm0.1}^{\dagger}$   & $\mathbf{77.86\pm0.2}^{\dagger}$   & $\mathbf{96.44\pm0.1}^{\dagger}$   & $\mathbf{80.56\pm0.5}^{\dagger}$\\
AdaGAT-n-3170    & $\mathbf{86.69\pm0.5}$                     & $\mathbf{76.57\pm1.0}$                      & $\mathbf{95.23\pm0.3}$                      & $\mathbf{78.64\pm0.5}^{\ddagger}$\\
\hline
\end{tabular}}
\caption{Average Precision \% for Multi-Task Learning. \textbf{First}$^{\dagger}$, \textbf{Second}$^{\ddagger}$, \textbf{Third} best performing method.}
\label{table:multi_task}
\end{center}
\end{table}

\vfill

Table \ref{table:temporal_graph} presents the results for future link prediction on dynamic graph. \adagnn\ outperforms the baselines in both transductive (2nd, 4th, 6th column) and inductive settings (3rd, 5th column) . One interesting model is AdaSage. For this model, we only aggregate the information from $1$ neighbor. The time complexity of this model is very low compared to the expensive dynamic graph method, such as TGN which needs to use RNN to update the memory at each training step. This simple model with boosting still achieves the accuracy around $94.5\%$. It gives us the idea that one can use boosting-based shallow models to achieve comparable accuracy using much shorter time and much smaller GPU memory.

Noticeably, \adagnn\ performs better on Movielens dataset than Wikipedia dataset. The reason is that \adagnn\ can utilize users' non-repetitive interaction behavior. In Wikipedia dataset, 69\% users keep editing the same page for the whole time domain, and 84\% users consecutively edit the same page. By 72\% chance, there is only one unique neighbor in node neighborhoods for a fixed sample size of $10$. In this case, the information of node neighborhoods is not very rich and projecting this one single neighbor onto multiple embedding spaces will not bring a dramatic increase. Movielens dataset, on the other hand, has 0\% users rating the same movie for the whole time domain and there are always more than one neighbors in node neighborhoods. In this case, the information of neighborhoods is rich and high-dimensional, so projecting neighborhoods onto multiple embedding spaces gives a large improvement. The difference between Wikipedia and Movielens datasets tells us that the idea of multiple embedding spaces in general performs better when the information of node neighborhoods is rich.

Table \ref{table:multi_task} presents the results on multi-task learning, including link prediction and node recommendation, in transductive (2nd, 4th column) and inductive settings (3rd, 5th column) . GAT-200, GAT-1024, GAT-3170 represent GAT with embedding dimension 200, 1024, 3170 respectively. \adagnn\ clearly outperforms the baselines by a large margin in both transductive and inductive settings for all dimensions, especially for the recommendation task.

\subsection{Model Analysis} \label{analysis}
In this section, we further verify the multiple embedding spaces experimentally, including
\begin{itemize}
\item Do multiple low-dimensional embedding spaces perform better than one single high-dimensional embedding space, even when the dimension of single space is equal to the sum of weak learners' dimensions?
\item Does increasing the number of embedding spaces bound the generalization error as we predict theoretically?
\item Do multiple embedding spaces capture different information in real experiment?
\item Is \adagnn\ robust towards limited training data?
\end{itemize}

\setlength\abovecaptionskip{1ex}
\setlength\belowcaptionskip{0ex}

\begin{figure}[H]
	\centering
	\includegraphics[scale=0.35]{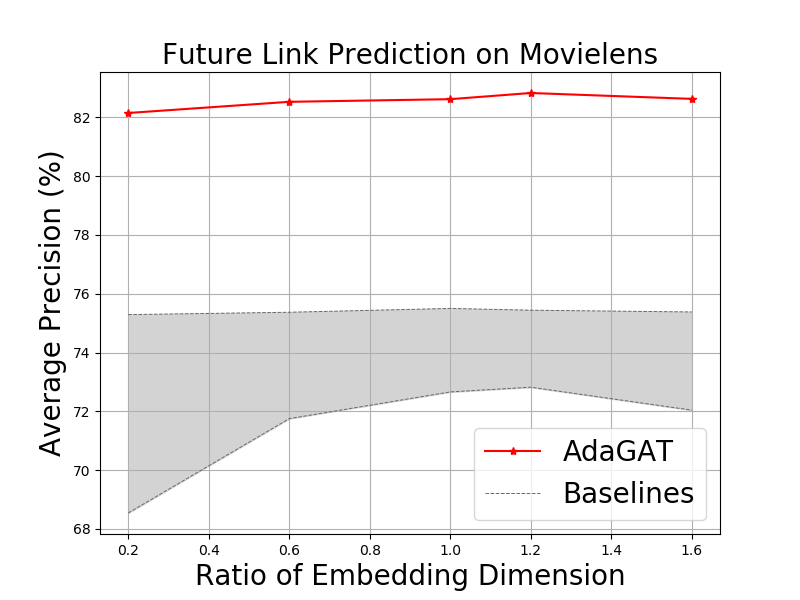}
	\caption{Varying the dimension of embedding space, default ($1.0$) is the dimension of node features.}
	\label{fig:embedding_dimension}
\end{figure}

\subsubsection{High-Dimensional Embedding Space}

For high-dimensional information of node neighborhoods, we mentioned before that it can be projected onto either one single high-dimensional embedding space or multiple low-dimensional embedding spaces. In Figure \ref{fig:embedding_dimension}, we discuss the performance of single high-dimensional embedding space and multiple low-dimensional embedding spaces by varying the embedding dimension from around 200 to 1800, and computing the average precision for static baselines and AdaGAT for different embedding dimensions on the future link prediction task of Movielens dataset. The effect on other tasks and datasets is similar. We observe that multiple embedding spaces outperform one single embedding space for all different embedding dimensions. Furthermore, AdaGAT even outperforms baselines when the embedding dimension of baselines and the sum of all subspace dimensions are almost equal. In these cases, it is difficult to find one single embedding space such that all pairs of linked nodes are close to each other and it's better to use multiple embedding spaces.

\setlength\abovecaptionskip{-2ex}
\begin{figure}
     \centering
     \begin{subfigure}[b]{0.15\textwidth}
         \centering
         \includegraphics[scale=0.15]{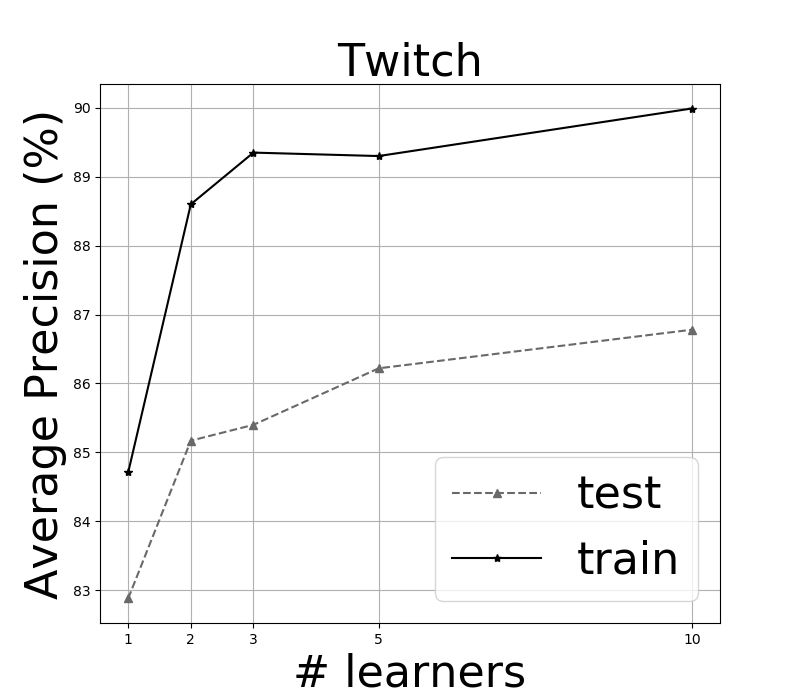}
         \label{fig:twitch_ap}
     \end{subfigure}
     \hfill
     \begin{subfigure}[b]{0.15\textwidth}
         \centering
         \includegraphics[scale=0.15]{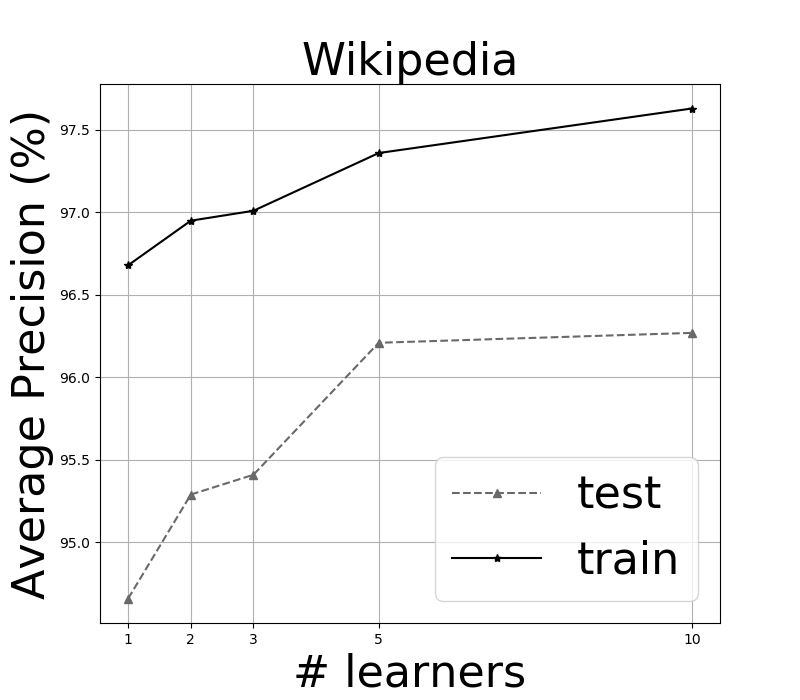}
         \label{fig:wikipedia_ap}
     \end{subfigure}
     \hfill
     \begin{subfigure}[b]{0.15\textwidth}
         \centering
         \includegraphics[scale=0.15]{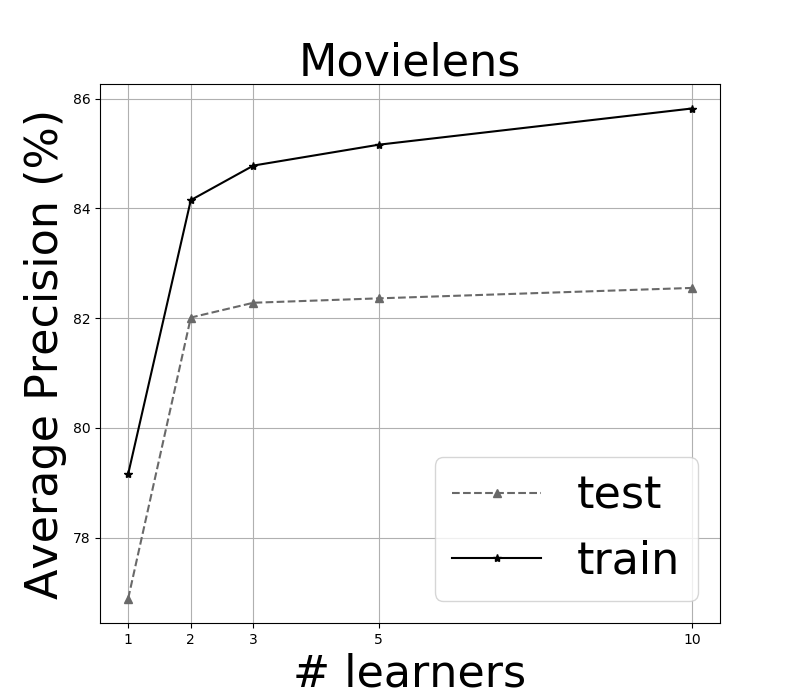}
         \label{fig:movilens_ap}
     \end{subfigure}

     \begin{subfigure}[b]{0.15\textwidth}
         \centering
         \includegraphics[scale=0.15]{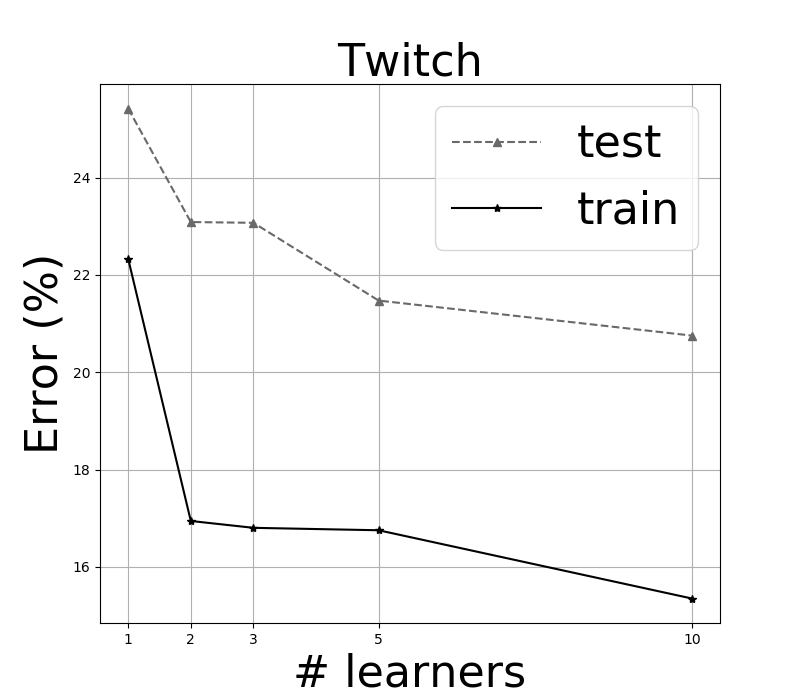}
         \label{fig:twitch_err}
     \end{subfigure}
     \hfill
     \begin{subfigure}[b]{0.15\textwidth}
         \centering
         \includegraphics[scale=0.15]{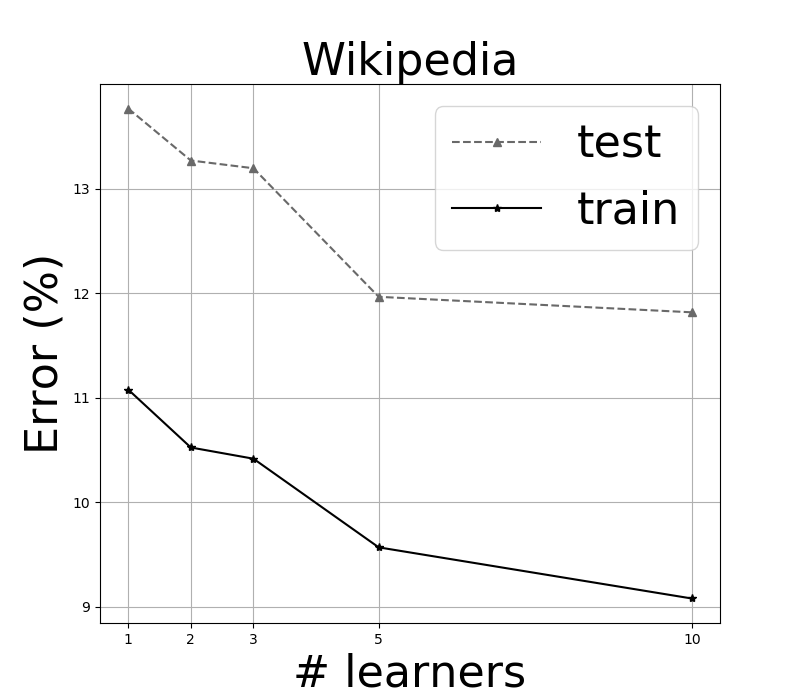}
         \label{fig:wikipedia_err}
     \end{subfigure}
     \hfill
     \begin{subfigure}[b]{0.15\textwidth}
         \centering
         \includegraphics[scale=0.15]{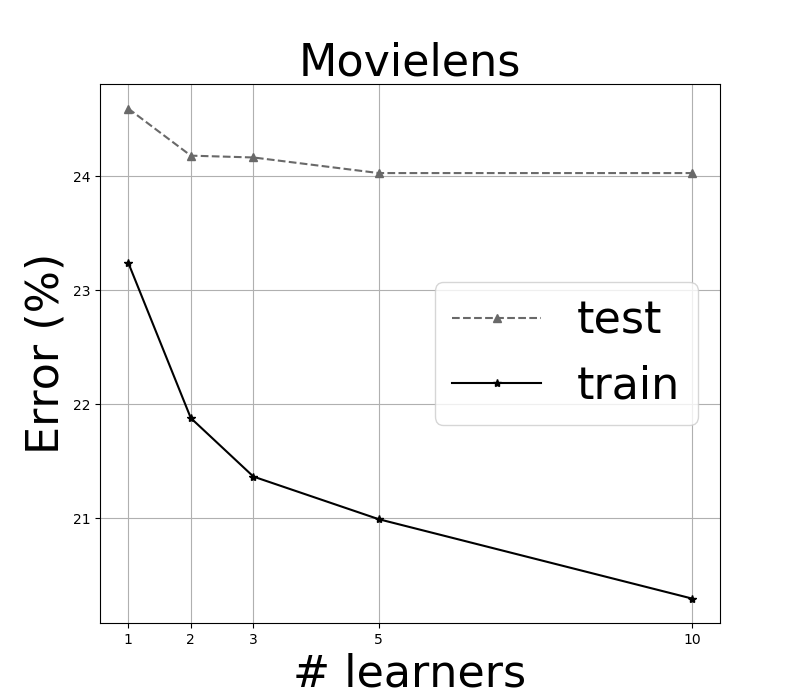}
         \label{fig:movilens_err}
     \end{subfigure}

    \begin{subfigure}[b]{0.15\textwidth}
         \centering
         \includegraphics[scale=0.15]{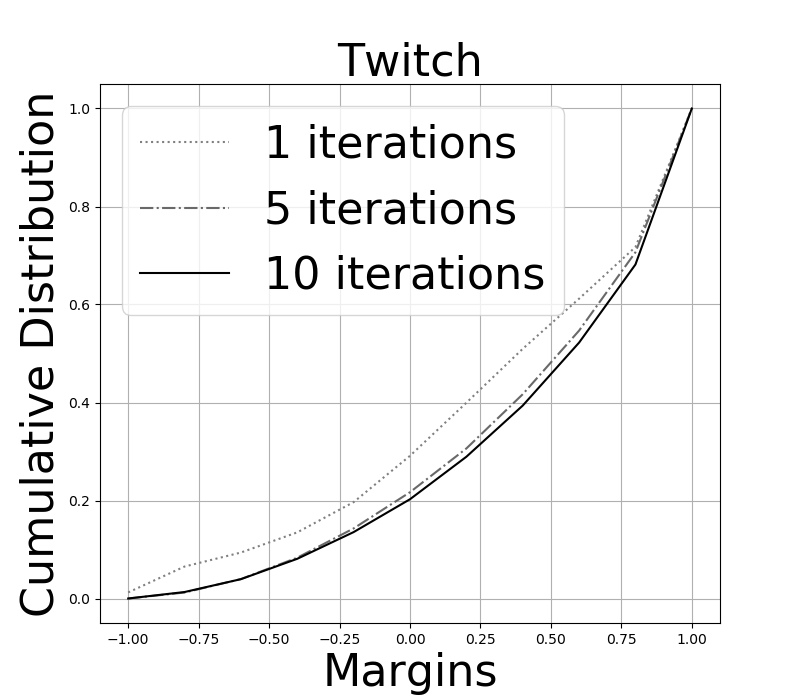}
         \label{fig:twitch_margin}
     \end{subfigure}
     \hfill
     \begin{subfigure}[b]{0.15\textwidth}
         \centering
         \includegraphics[scale=0.15]{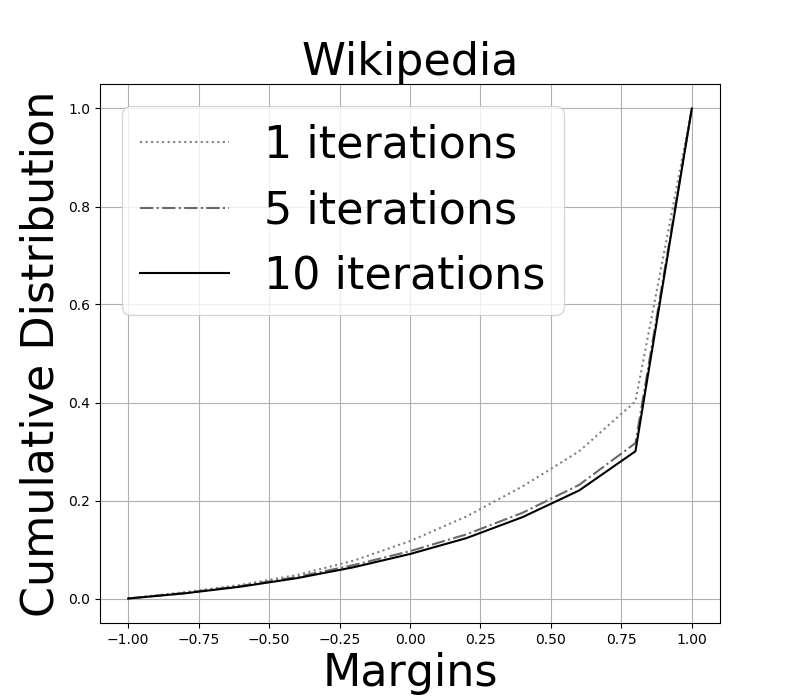}
         \label{fig:wikipedia_margin}
     \end{subfigure}
     \hfill
     \begin{subfigure}[b]{0.15\textwidth}
         \centering
         \includegraphics[scale=0.15]{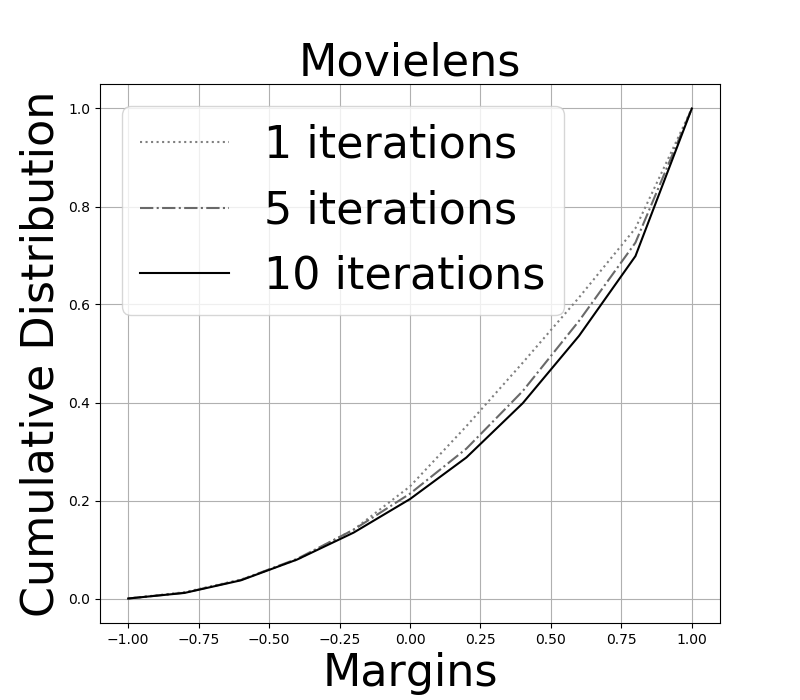}
         \label{fig:movilens_margin}
     \end{subfigure}
    
     \hfill
     \caption{Average precision, error curves and margin distribution graphs for boosting-based GAT on three datasets. The first, second and last lines are average precision, error curves and margin distribution graphs respectively.}
     \label{fig:generalization}
\end{figure}

\begin{figure*}
     \centering
     \begin{subfigure}[b]{0.24\textwidth}
         \centering
         \includegraphics[scale=0.3]{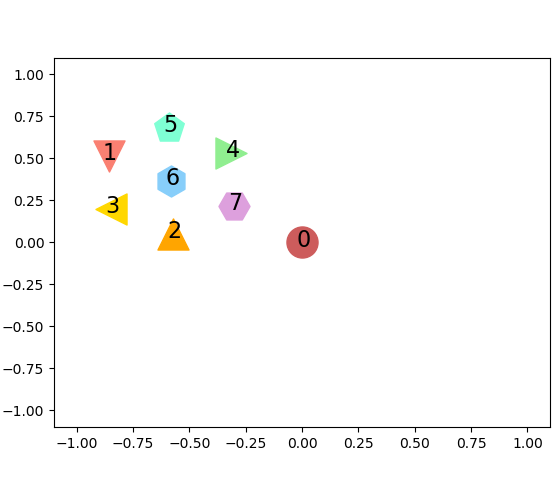}
         \label{fig:movielens_embedding1}
     \end{subfigure}
     \hfill
     \begin{subfigure}[b]{0.24\textwidth}
         \centering
         \includegraphics[scale=0.3]{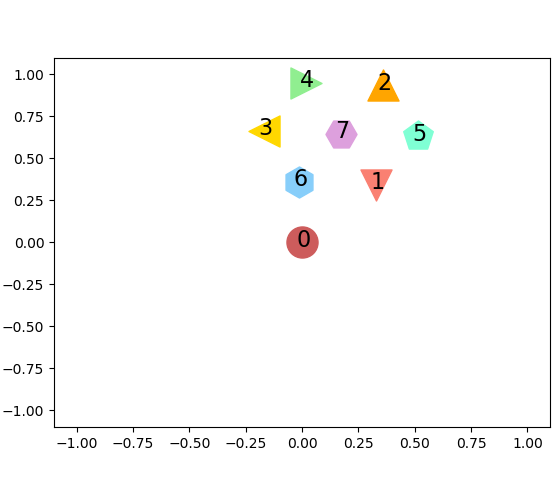}
        %  \caption{node closest}
         \label{fig:movielens_embedding2}
     \end{subfigure}
     \hfill
     \begin{subfigure}[b]{0.24\textwidth}
         \centering
         \includegraphics[scale=0.3]{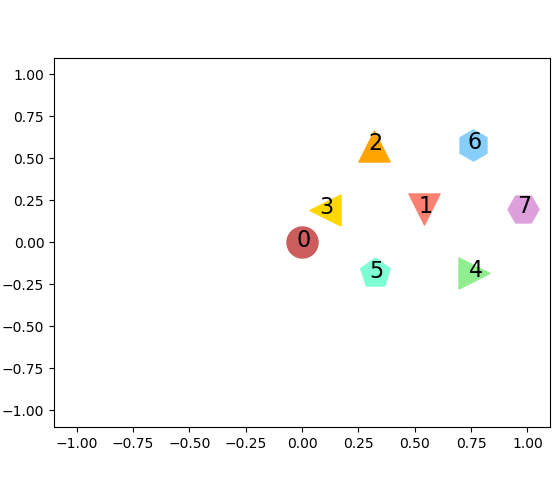}
        %  \caption{Embedding space}
         \label{fig:movielens_embedding3}
     \end{subfigure}
     \hfill
     \begin{subfigure}[b]{0.24\textwidth}
         \centering
         \includegraphics[scale=0.3]{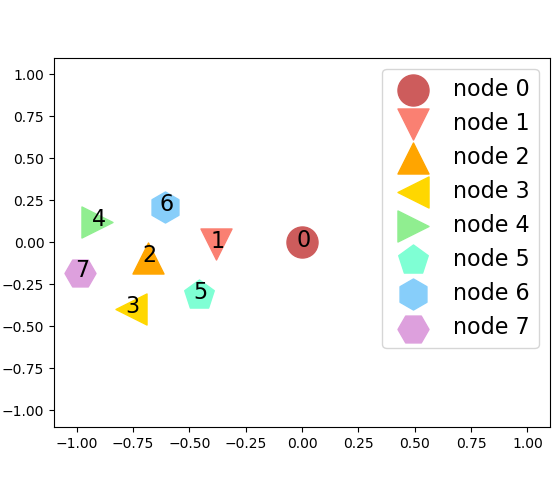}
        %  \caption{Embedding space}
         \label{fig:movielens_embedding4}
     \end{subfigure}
     \hfill
     \caption{Visualization of multiple embedding spaces for multiple learners. Node 1-7 are the neighbors of node 0. We move the node 0 to the origin and rescale the maximum distance between node 0 and node 1-7 to be one.}
     \label{fig:movelens_embedding_space}
\end{figure*}

\subsubsection{Generalization error}

We use the margin theory to analyze the generalization errors of \adagnn. Margins can be considered as a measure of confidence. Here we define the margin as $(y_{i,j}-\tau)(f(\mathcal{G}_i,\mathcal{G}_j)-\tau)$, where $y_{i,j}\in[0, 1]$ represents the existence of edges or not, $f(\mathcal{G}_i,\mathcal{G}_j)\in[0, 1]$ represents the predicted similarity between two nodes and $\tau\in(0, 1)$ is the threshold. The margin is negative when the prediction is incorrect and positive otherwise. Similar to \cite{Scarselli18}, the margin reaches minimum when the model predicts incorrectly with high confidence, maximum when the model predicts correctly with high confidence, and close to zero when the prediction has low confidence.  For example, given $y_{i,j}=1, \tau = 0.5$, the margin reaches minimum if $f(\mathcal{G}_i,\mathcal{G}_j)=0$, maximum if $f(\mathcal{G}_i,\mathcal{G}_j)=1$, and close to zero if $f(\mathcal{G}_i,\mathcal{G}_j)=0.501$. Therefore, margin is a measure of correctness and confidence. When the prediction is based on a clear and substantial majority of the base classifiers, the margin will be positively large, corresponding to greater confidence in the predicted labels.

%We have proved that margins on the training data guarantee an upper bound on the testing error in Section \ref{margin_theory}. Therefore, we can bound the generalization errors of \adagnn\ by verifying that \adagnn\ leads to larger margins for training data.% In the other word, we will validate that boosting tends to become more confident in its own predictions, and the greater the confidence in a prediction, the more likely it is to be correct. 

In Section \ref{margin_theory}, we showed the existence of the generalization error theorem when the training and test errors are measured using the margin metric. Therefore, under margin metric, we can bound the test error by looking at the training error and the generalization error bound.

We visualize the effect of boosting on the margins by plotting their distribution, as in Figure \ref{fig:generalization}. The first two rows are the average precision scores and errors for different number of weak learners on Twitch, Wikipedia, Movielens datasets respectively. The last row contains the margin distribution graph. We can observe that boosting aggressively pushed up the training data with small or negative margin, so the number of training data with small or negative margin decreased, which leads to smaller training and testing errors. The generalization error bound in Eq. \ref{eq:bound} in terms of number of samples $m$ and VC-dimension $d$ has no explicit dependence on the number of weak learners. In the second row of Figure \ref{fig:generalization} the observed generalization errors in reality also stays roughly constant with respect to the number of weak learners.

\subsubsection{Embedding Visualization} 
For different weak learners, we want that it projects the node neighborhood  $\mathcal{G}_i$ onto different embedding spaces, capturing different similarities. In order to visualize it, we plot node embeddings in different spaces, as in Figure \ref{fig:movelens_embedding_space}. We use t-SNE \cite{laurens08} to visualize the high-dimensional embeddings in 2d space. In Figure \ref{fig:movelens_embedding_space}, node 1-7 are the neighbors of node 0 and we plot their embeddings in four different embedding spaces. In embedding space 1, node 7 is closest to node 0 and node 1 is farthest from node 0, while node 1 is closest to node 0 and node 7 is farthest from node 0 in embedding space 4. Similarly, node 6 is closest to node 0 in embedding space 2 while it is far away from node 0 in embedding space 3. From these observations, we can see that for the same pair of nodes, their similarity is different from space to space. It experimentally verifies that each embedding space captures different information and preserves different similarities about node neighborhoods in this case.

Combine the results in Figure \ref{fig:embedding_dimension} and \ref{fig:movelens_embedding_space}, we can experimentally validate the necessity of multiple embedding spaces. In some case, it's hard to find one single embedding space such that two linked nodes are always close to each other. Therefore, we can use multiple embedding spaces and only require parts of linked nodes are close to each other in each embedding space, which is easier to achieve. Then we can combine the similarities in all embedding spaces to predict the existence of edges.

%\subsubsection{Multi-Task Learning}
%multi-task learning v.s. single-task learning
\subsubsection{Robustness to Limited Training Data}

One challenge of learning on social networks is the lack of high-quality data. Therefore, an ideal model should be efficient in leveraging limited training data. In this experiment, we validate the robustness of AdaGAT to limited training data in Figure \ref{fig:training_data}. We vary the ratio of training data from $0.1$ to $0.9$ on Movielens dataset. The numbers of validation and testing data are always equal. We can see that the AdaGAT always outperforms the baselines. %but the gap becomes larger as the number of training data increase.

\setlength\abovecaptionskip{1ex}
\setlength\belowcaptionskip{0ex}
\begin{figure}[H]
    \centering
    \includegraphics[scale=0.35]{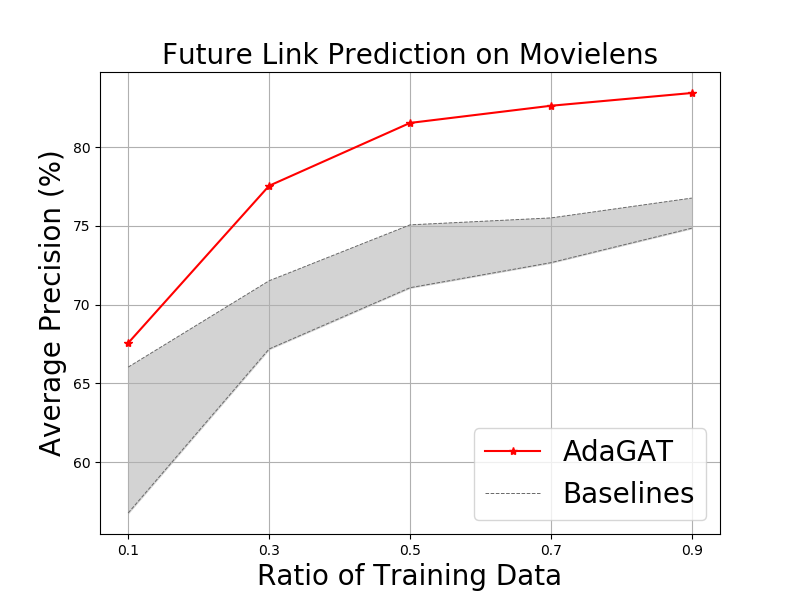}
    \caption{Varying the number of training data.}
    \label{fig:training_data}
\end{figure}

\section{Conclusion}

In this work, we introduce a novel approach to automatically project the node neighborhoods onto multiple low-dimensional embedding spaces using boosting method and develop \adagnn\ based on this approach. We theoretically and experimentally analyze the effectiveness and robustness of \adagnn, especially the advantages of multiple embedding spaces over single embedding space. We demonstrate that \adagnn\ can achieve great performance when the information of node neighborhoods is rich and the dimension of the ideal embedding space is large. Our work envisions the novel application of multiple embedding spaces and boosting method in graph neural network, opens up a direction along the idea that preserves the similarities between nodes in different spaces in the field of social networks and also leaves us several future questions to think about, including how to further reduce the information leakage between embedding spaces, how to further narrow down the focus of each embedding space and how to further measure the richness of node neighborhoods.  

\newpage
\bibliographystyle{ACM-Reference-Format}
\bibliography{reference}

\end{document}